\documentclass[runningheads]{llncs}

\usepackage{times}

\usepackage{ucs}
\usepackage[utf8x]{inputenc}
\usepackage{amsmath}
\usepackage{amsfonts}
\usepackage{amssymb}
\usepackage[american]{babel}
\usepackage{fontenc}
\usepackage{graphicx}\graphicspath{{./figs/}}
\usepackage{url}

\usepackage{mdwlist}
\usepackage{enumerate}

\usepackage{stackrel}
\usepackage{multirow}

\usepackage{tikz}

\newcommand{\setR}{\mathbb{R}}
\providecommand{\inarcs}[1]{\delta^{in}(#1)}
\providecommand{\outarcs}[1]{\delta^{out}(#1)}
\providecommand{\abs}[1]{\lvert #1 \rvert}
\providecommand{\Dunit}{\text{ m}^2\text{s}^{-1}}

\begin{document}

\title{Blinking Molecule Tracking}

\author{Andreas Karrenbauer\inst{1}\thanks{Supported by the Max-Planck-Center for Visual Computing and Communication}
\and Dominik Wöll\inst{2}\thanks{Supported by the Zukunftskolleg, University of Konstanz, Germany}}
\institute{Max-Planck-Institute for Informatics, Saarbrücken, Germany \email{andreas.karrenbauer@mpi-inf.mpg.de} \and Department of Chemistry, University of Konstanz, Germany \email{dominik.woell@uni-konstanz.de}}

\maketitle

\begin{abstract}
We discuss a method for tracking individual molecules which globally optimizes the likelihood of the connections between molecule positions fast and with high reliability even for high spot densities and blinking molecules. Our method works with cost functions which can be freely chosen to combine costs for distances between spots in space and time and which can account for the reliability of positioning a molecule. To this end, we describe a top-down polyhedral approach to the problem of tracking many individual molecules. This immediately yields an effective implementation using standard linear programming solvers.
Our method can be applied to 2D and 3D tracking.
\end{abstract}

\section{Introduction}

The possibility to observe single fluorescent molecules in real-time
has opened up a lot of new insights into the dynamics of systems in biology and
material sciences. Single molecule microscopy (SMM) allows for the parallel
observation of translational and rotational motion of many single fluorescent
molecules beyond the diffraction limit provided that their concentration is reasonably low.
However, tracking of
single fluorescent molecules bears the challenge that the fluorescent spots are
rather weak with a low signal/noise-ratio and show significant changes in signal
intensity \cite{Cichos_2007_15,Lippitz_2005_16,Bingemann_2006_12}.
In the extreme case, a fluorescent molecule is dark for several recorded frames,
a phenomenon which is termed blinking \cite{Goehde_1998_29,Yip_1998_30}.

The fact that the fluorescence signals of single molecules cannot be classified
according to their intensity or shape in different frames and even disappear in
some frames causes severe problems for single molecule tracking. Thus, many
tracking algorithms which have been developed for single particle tracking (e.g.~for cells)
in video microscopy fail for tracking single molecules.

The path from the recorded single molecule microscopy movies to the results about
their motion includes typically the following steps:
\begin{enumerate}[(i)]
\item determination of the positions of each fluorescent molecule,
\item connecting the positions to single molecule tracks, and
\item statistical analysis of these tracks.
\end{enumerate}

\subsection{Previous Work}

Nowadays, single molecule positions
are most often determined using center of mass or Gaussian fits, with the latter
being the best option for low signal-to-noise ratios of around four \cite{Cheezum_2001_1}.
Usually the images are preprocessed by various filters, e.g.~\emph{Mexican hat} \cite{Sage_2005_33}, before the actual localization.

After localization, the positions of subsequent frames
have to be connected to tracks \cite{Schmidt_1996_3}. Different approaches have been
developed for this purpose, but up to now it remains challenging to improve and
develop algorithms not only for special tasks but for a
universal set of problems~\cite{Saxton_2008_10}.

Single Particle Tracking procedures started with the connection of one point
with its closest neighbor in consecutive frames \cite{Ghosh_1994_24}.
In 1999, Chetverikov et al.
published a new algorithm called IPAN Tracker~\cite{Chetverikov_1999_35}.
Using a competitive linking process that develops as the
trajectories grow, this algorithm deals better with incomplete trajectories,
high spot densities, faster moving particles and appearing and disappearing
spots. Sbalzarini et al. used the same approach, but did
not make any assumptions about the smoothness of trajectories \cite{Sbalzarini_2005_34}.
Their algorithm was implemented as ParticleTracker in ImageJ.

The SpotTracker \cite{Sage_2005_33} is a very powerful tool to follow single spots
throughout one movie, but it can only proceed spot by spot. The algorithm proposed by Bonneau et al.~\cite{Bonneau_2005_57} falls in the same category of greedy algorithms that iteratively compute shortest paths in space-time, which are not revised subsequently.

One of the most accurate solutions to single particle tracking is provided by multiple-hypothesis tracking (MHT).
This method chooses the largest non-conflicting ensemble of single particle paths
simultaneously accounting for all position in each frame. Jaqaman et al.~used
such an approach where they first linked positions in consecutive frames by solving bipartite matching problems and
combined these links into entire trajectories~\cite{Jaqaman_2008_2} with a post-processing step to account for missing points in a frame.
Both steps were optimized independently yielding a very likely solution of the
tracking problem. Dynamic multiple-target tracing was used by Sergé et al.~to
generate dynamic maps of tracked molecules at high density~\cite{Serge_2008_48}.
Subtracting detected peaks from the images allows them for a detection of low
intensity peaks which would be otherwise hidden in movies of high particle
density.
Peak positions were connected using statistical information from past
trajectories.

Moreover, manual or semi-automated approaches, which only perform unambiguous choices automatically, are still used though there are cumbersome due to many user interaction at high particle densities.

For the analysis of the tracks, different approaches have been
developed \cite{Elliott_2011_54,Saxton_2008_10}. The most common approach is the
analysis of the mean squared displacement for different time
intervals \cite{Hellriegel_2005_53,Schmidt_1995_50}
which can readily distinguish between different modes of motion such as
normal diffusion, anomalous diffusion, confined diffusion, drift and active
transport~\cite{Saxton_1997_11}.
Alternatively, the empirical distribution of squared
displacements \cite{Schuetz_1997_49} and radii of gyration~\cite{Rudnick_1987_56,Werley_2006_55,Elliott_2011_54}
can be used to analyze single molecule tracks.

In a compagnion paper~\cite{WKSK2013}, we report on the implications of our work from a chemical point of view. Whereas in this paper, we highlight the algorithmic aspects of our approach.

\subsection{Our Contribution}

We present a method for single molecule tracking which globally
optimizes the likelihood of the connections between molecule positions fast and with high
reliability even for high spot densities. Our method uses cost functions which
can be freely chosen to combine costs for distances between spots in space and
time and which can account for the reliability of positioning a molecule.
Using a suitable positioning procedure, reliable tracking can be performed even
for highly mobile, frequently blinking and low intensity fluorescent molecules,
cases for which most other tracking algorithms fail.

In the following, we present a top-down approach for modeling molecule tracking. 
We thereby unify the previous approaches in one framework. By developing a suitable polyhedral model in Sec.~\ref{sec:model}, we show theoretically that it remains computationally tractable. A major advantage of our method is that we immediately obtain an effective software solution using standard linear programming software. Moreover, we experimentally evaluate our implementation in Sec.~\ref{sec:experiments}. To this end, we use real-world data and realistic data, i.e.~randomly generated according to a physical model. We qualitatively compare our tracking on the real-world data to tracks obtained by a human expert, whereas we exploit the knowledge about the ground-truth in the realistic data to quantitatively measure the impact of noise and the validity of the parameters that we have chosen. We evaluated our approach for two-dimensional tracking, but the extension to 3D is straight-forward. We provide our software as open source code\footnote{\url{http://arxiv.org/src/1212.5877v2/anc/tracking.m}} for MATLAB using CPLEX as LP-solver.

\section{A Polyhedral Model for Molecule Tracking}\label{sec:model}

It is easy to see that the number of possible trajectories grows exponentially with the number of points. To tackle this \emph{combinatorial explosion} \cite{Bonneau_2005_57}, we consider the a top-down polyhedral approach for a concise representation in this paper. Suppose we are given a set of points $V = \{v_1, \ldots, v_n\}$. Each point has one temporal and $d$ spatial coordinates, say $(t_i,x_i,y_i)$, $i \in V$, with $d=2$ as used in the following for the sake of presentation.
We postulate the following conditions for a track:
\begin{itemize*}
 \item Each point has at most one predecessor.
 \item Each point has at most one successor.
\end{itemize*}
We model the predecessor/successor relation of two points by ordered pairs. To this end, let $V_{t < t_j} := \{ v_i \in V: t_i < t_j \}$ and $A := \{(v_i,v_j): v_i \in V_{t < t_j},\, v_j \in V \}$. We denote the predecessor/successor relation by $f: A \to \{0,1\}$. Moreover, let $g,h: V \to \{0,1\}$ denote missing predecessors and successors, respectively. That is, $g(v) = 1$ iff $v \in V$ does not have a predecessor, and $h(v) = 1$ iff it does not have a successor. Let $\inarcs{v},\outarcs{v} \subseteq A$ denote the sets of possible predecessor/successor relations for point $v$.

\begin{definition}
A \textbf{track partition} or \textbf{tracking} of $V$ is a collection of disjoint tracks covering $V$, i.e.~each point appears in exactly one track, which might consist of a single point. 
\end{definition}

The characteristic vector $(f,g,h)$ of a track partition is a $\{0,1\}$-vector in which the first $\abs{A}$ entries corresponding to $f$ denote the predecessors/successor relation, the following $\abs{V}$ entries corresponding to $g$ determine the starting points of the tracks, and the last $\abs{V}$ entries corresponding to $h$ define the endpoints of the tracks.

\begin{theorem}\label{thm:polytope}
The \textbf{tracking polytope}, i.e.~the convex hull of all track partitions, is given by
\[
P := \{ (f,g,h) \in \setR_{\ge 0}^{\abs{A}+2\abs{V}} : \forall v \in V:  g(v) + f(\inarcs{v}) = 1, h(v) + f(\outarcs{v}) = 1\}.
\]
\end{theorem}

\begin{proof}
It is easy to see that the characteristic vector of a tracking is contained in $P$. Moreover, each $\{0,1\}$-vector in $P$ corresponds to a tracking. Hence, it remains to show that these are the only vertices of $P$. To this end, we prove that the constraint matrix $M$ that defines $P$ in the form $P = \{ x : M x = 1, x \ge 0 \}$ is totally unimodular. First, we observe that $M = (M' I)$ where $M'$ corresponds to the $f$-variables. Hence, it suffices to show total unimodularity for $M'$. To this end, we consider an auxiliary graph $G'=(V_1 \dot{\cup} V_2, E')$ at which each of $V_{1,2}$ contains a copy of each point in $V$ and $E'$ is the set of edges that mimics the set $A$ on $V_1 \times V_2$. Note that by this definition $G'$ is bipartite. Moreover, its adjacency matrix is given by $M'$. Hence, $M'$ and $M$ are totally unimodular.
\end{proof}

\subsection{Optimization}

Based on the compact representation of all possible tracks as described before, we now consider the problem of selecting an appropriate tracking out of all these possibilities. To this end, we leverage the fundamental paradigm of normal diffusion: Tracks are Markov chains, i.e.~the transition probability from one state to another does only depend on the current state and not on the history that led to it. Thus, all transitions are independent random events.

Suppose that we are given probabilities $p_1: A \to [0,1]$ for the transitions and $p_{2,3}: V \to [0,1]$ denoting the probability that a point is the beginning or the end of a track, respectively.
We wish to find a tracking with maximum likelihood, i.e.~a tracking that maximizes the joint probability of the independent random events
\[
\mathcal{L}(f,g,h) = \prod_{\stackrel{a \in A}{f(a)=1}} p_1(a) \prod_{\stackrel{v\in V}{g(v)=1}}p_2(v) \prod_{\stackrel{v\in V}{h(v)=1}}p_3(v)
\]
or, equivalently, 
\[
\log \mathcal{L}(f,g,h) = \sum_{\stackrel{a \in A}{f(a)=1}} \log p_1(a) + \sum_{\stackrel{v\in V}{g(v)=1}} \log p_2(v) + \sum_{\stackrel{v\in V}{h(v)=1}} \log p_3(v).
\]
Hence, by substituting $c_i = -\log p_i \ge 0$, finding the most likely tracking amounts to solve the linear programming problem
\begin{equation}\label{eq:LP}
 \min \left\{\sum_{a \in A} c_1(a) f(a) + \sum_{v\in V} c_2(v) g(v) + \sum_{v\in V} c_3(v) h(v) : (f,g,h) \in P \right\},
\end{equation}
where we exploit the consequence of Thm.~\ref{thm:polytope} that the minimum is attained by a $\{0,1\}$-solution for $(f,g,h)$. Put differently, it is not necessary to enforce an integer solution by Integer Linear Programming, which is NP-hard in general, but it is sufficient to solve the LP-relaxation~\eqref{eq:LP}, which can be done in polynomial time using the ellipsoid method~\cite{Khachiyan79} or interior point methods~\cite{Karmarkar84,Ye91}.
Note that the formulation~\eqref{eq:LP} is general enough to capture arbitrary separable likelihood functions $\mathcal{L}(f,g,h)$.

\begin{lemma}\label{lem:radius}
For all optimum solutions $(f,g,h) \in P$ and $a=(v,w) \in A$, we have
\[
 f(a) = 1 \quad \Rightarrow \quad c_1(a) \le c_2(v) + c_3(w).
\]
\end{lemma}
\begin{proof}
By contradiction: We would obtain a better feasible solution by setting $f(a)=0$ and $g(v) = h(w) = 1$.
\end{proof}

\noindent Although it is possible to consider different probabilities for appearing and vanishing particles, we choose a constant one, i.e. let $c_2(v) = c_3(v) = C$ for all $v \in V$. Hence, $c(a) \le 2\cdot C$ for all $a \in A$ which appear in any optimum solution. This inspires the definition of a \textbf{tracking radius} $R$ such that we will only consider predecessors and successors within that range. This dramatically limits the size of an instance and enables us to use space partition techniques to efficiently construct the tracking LP. Put differently, the restriction to a certain tracking radius for efficiency reason is justified by Lem.~\ref{lem:radius}. In the experiments section, we will discuss suitable choices for $R$. Similarly, it makes sense to limit the number of frames that a molecule might be invisible.

\subsection{Dealing with Noise}

Since the points are usually detected from noisy images, there are false positive and false negatives. That is, a spot $v \in V$ is a false positive, if it does not correspond to any track. A false negative is a point that does not have a correspondent in $V$.

Moreover, there might be $v,w \in V$ that correspond to the same point. To deal with these duplicates, we introduce so-called \emph{joins} into $A$. That is, we allow that two points from the same frame appear in one track. We thereby  maintain the integrality of our polyhedron. However, we treat these joins differently w.r.t.~the objective function to reflect the special situation. We propose to set the  cost of such a link to the ordinary cost of that connection plus the mean of the penalties for not having a successor and a predecessor, respectively. Taking the penalty into account is necessary to avoid 2-cycles.

Since we can only deal with points that are present in $V$, we shall avoid false negatives in the point detection. However, a low false negative rate often leads to a high false positive rate. Therefore, we utilize the possibility to consider a quality measure of each detected point. That is, we reduce the cost of not tracking a point according to its quality. This can be modeled easily by multiplying $c_{2,3}(v)$ by some quality factor $q(v)$, e.g.~proportional to the strength of the signal of this spot.

\section{Experiments}\label{sec:experiments}

So far, we described a generic approach for tracking blinking molecules. In this section, we propose our choice for the cost-function, i.e.
\[c(v_i,v_j) = (x_i - x_j)^2 + (y_i - y_j)^2 + (t_i - t_j)^2,\] which is validated experimentally.
The rationale for using this function is based on the following observation:
if time is not penalized, then track fragmentation becomes more likely as shown in Fig.~\ref{fig:costs}.

\begin{figure}[ht]
 \centering
 \includegraphics[height=30mm,page=1]{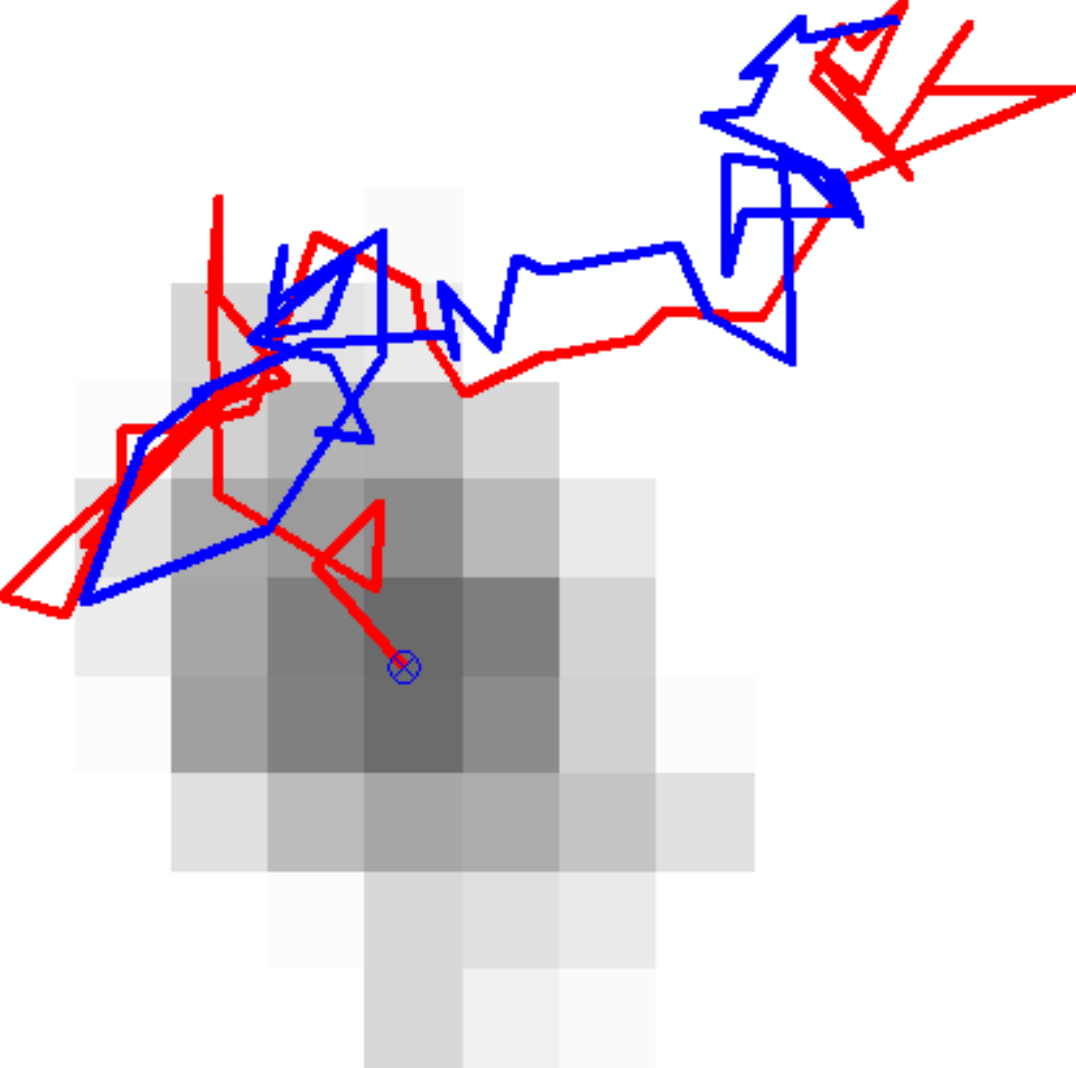}
 \hspace*{20mm}
 \includegraphics[height=30mm,page=2]{figures.pdf}
 \caption{Tracking with $\Delta x^2 + \Delta y^2$ (left) and $\Delta x^2 + \Delta y^2 + \Delta t^2$ (right). Two interleaved tracks are produced for one molecule on the left, whereas there is exactly one on the right. The background images are inverted for better readability.}
 \label{fig:costs}
\end{figure}

Thus, we introduced the superlinear term $\Delta t^2$ such that two time steps
of length 1 are cheaper than one time step of length 2. Hence, the spatial
distance is mainly responsible for comparing positions within the same frame.
We consider closer destinations to be more likely. Therefore, we do not use
the time in this part of the objective.

We evaluated our approach w.r.t.~efficiency and accuracy (the latter is only discussed briefly in this paper to the extent that is relevant for algorithmic conclusions and a more detailed analysis, in particular w.r.t.~chemistry, is presented in~\cite{WKSK2013}).
We first consider a controlled testing environment based on the normal diffusion model mentioned above. We thereby obtain realistic randomly generated instances. Though the experiments with the simulated realistic data has the advantage that we know the ground truth and thus we can quantify the deviation of the computed results in certain situations, we shall also validate our approach on real-world instances. To this end, we compare the diffusion coefficients obtained manually by a human expert with our automated approach in the final subsection of this paper.

\subsection{Realistic Data}

Synthetic trajectories were generated to test our tracking approach. Random walk simulations with 500 steps were performed using four
different diffusion coefficients ($10^{-15}$, $10^{-14}$, $10^{-13}$, and $10^{-12}\Dunit$), five
different signal-to-noise ratios (1, 2, 3, 4, and 5, see Fig.~\ref{fig:SN}) and five different particle
densities (100, 200, 300, 400 and 500 particles per frame). These values can be
found in many tracking challenges of practical importance, but were also chosen
to determine the sensitivity of our algorithms.

\begin{figure}[ht]
 \centering
 \includegraphics[width=0.9 \textwidth,page=3]{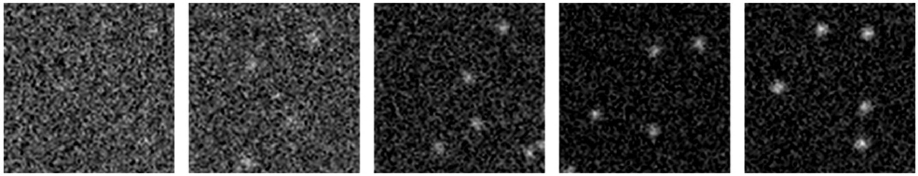}
 \caption{Different signal-to-noise ratios (1,2,3,4, and 5 from left to right).}
 \label{fig:SN}
\end{figure}

The initial x- and y-position of
each spot was chosen uniformly at random in the range [0;500]. The single
molecule trajectories were used to construct movies with 500 frames. Two
dimensional Gaussian functions were constructed with their center located at the
positions obtained from the random walk simulations and their widths being
diffraction limited (ca.~300 nm). We did not explicitly vary the intensities of the spots as it might appear in practice because testing the localization routine is not our primary focus in this paper.\footnote{Note that our approach is modular such that any localization procedure may be used.} However, we implicitly simulated variations in the intensities due to noise and blinking.
The time between consecutive frames was chosen
as 0.1 s and the resolution as 100 nm per pixel. Gaussian white noise was added
to the movies corresponding to the signal-to-noise ratio defined as   with the
amplitude of the center of the 2D Gaussian and the standard deviation of the
Gaussian white noise.

\subsubsection{Evaluation of Running Times}

The running times were measured on a Dell Precision T7500 with an Intel Xenon CPU X5570 at 2.93 GHz and 24 GB RAM.
The memory usage never exceeded 2 GB. In the following, we will discuss the scaling behavior of the running time. That is, the CPU time used for constructing the constraints using a space partition\footnote{We use a regular grid with spacing equal to the tracking radius.} and for solving the LP with CPLEX's barrier interior point method, which turned out to perform best to solve such problems from scratch. In Fig.~\ref{fig:runningtimes}, we present the scaling behavior in dependence on the particle density.

\begin{figure}[ht]
 \centering
 \includegraphics[width=0.325 \textwidth,page=4]{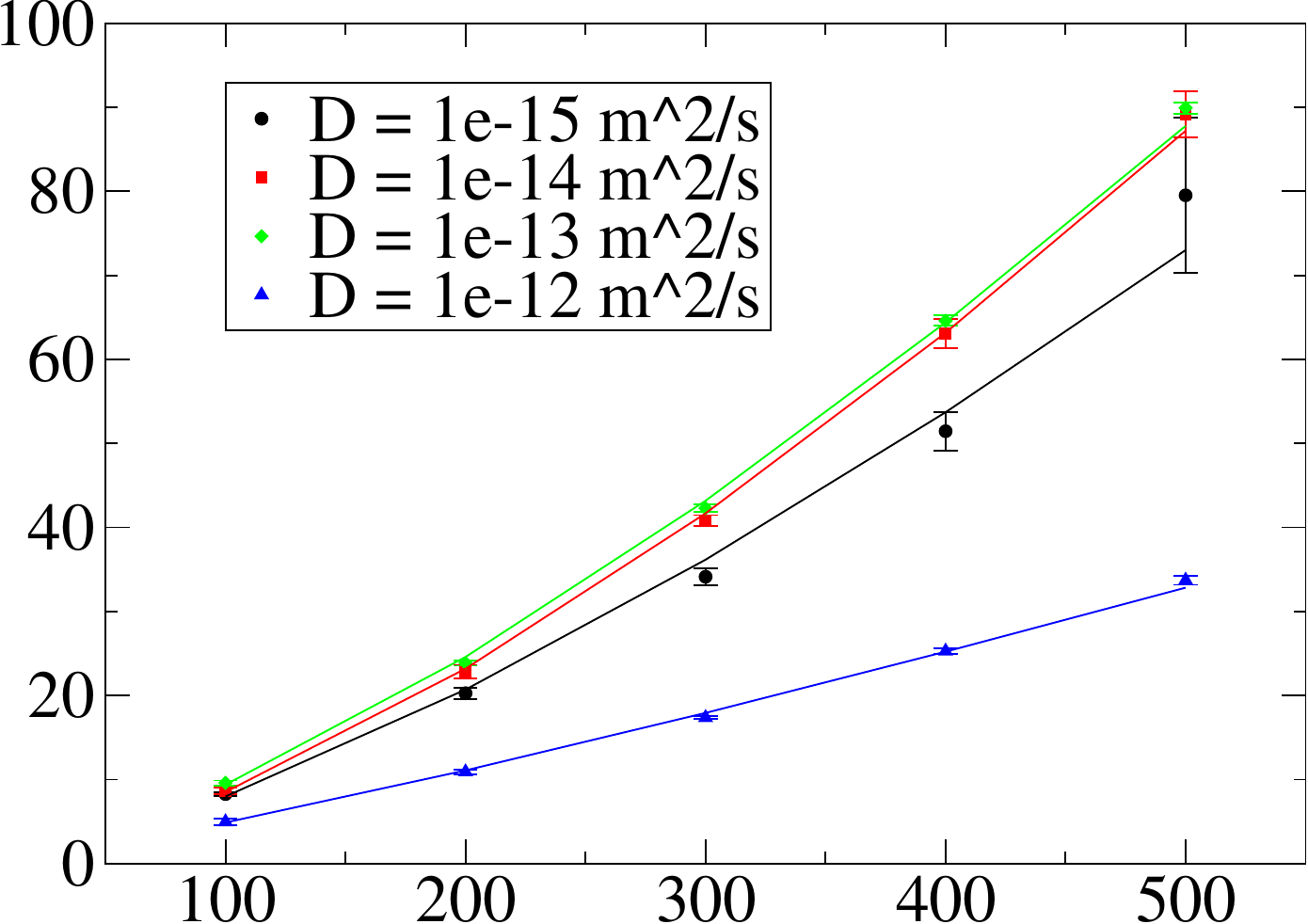}
 \hfill
 \includegraphics[width=0.325 \textwidth,page=5]{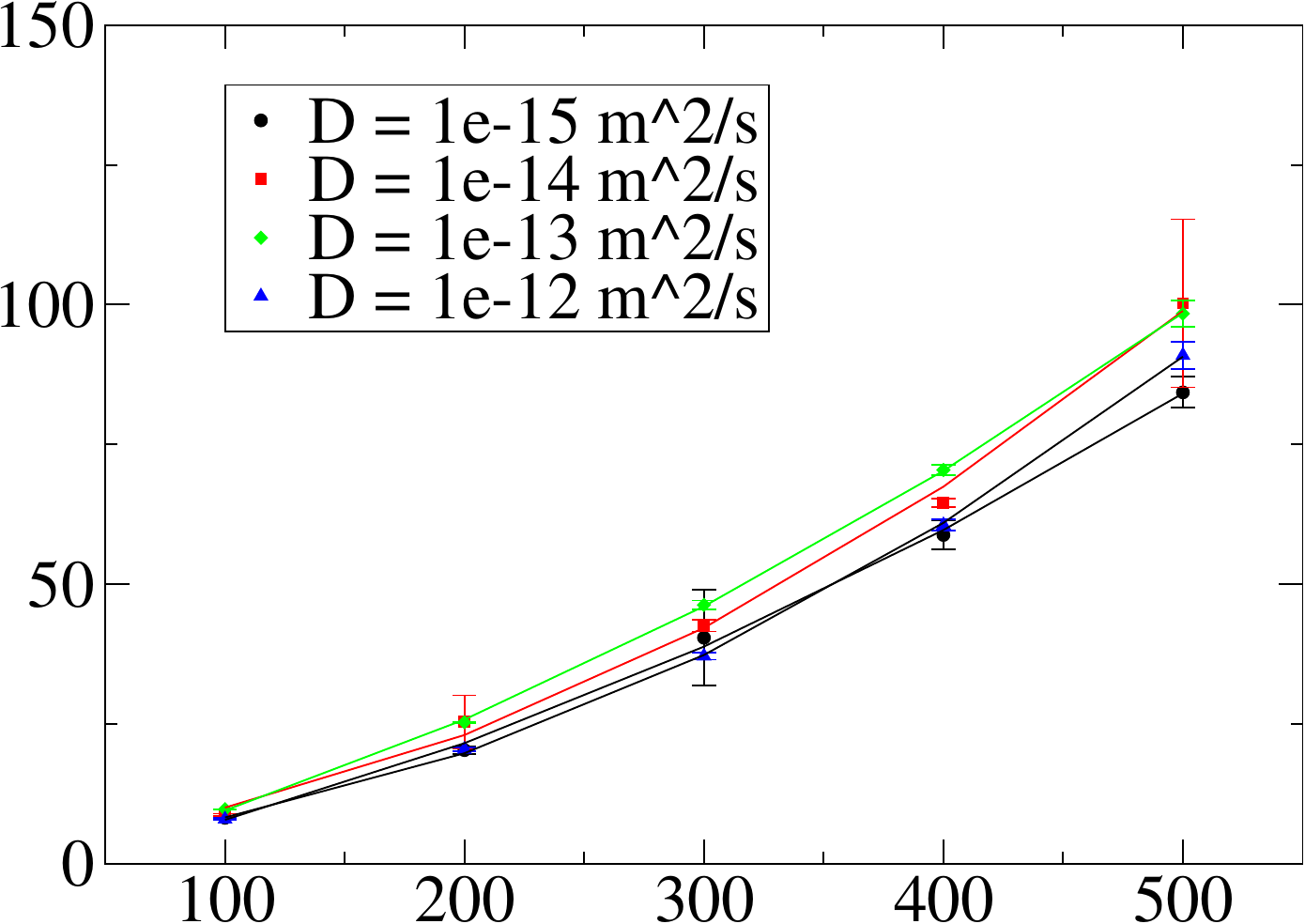}
 \hfill
 \includegraphics[width=0.325 \textwidth,page=6]{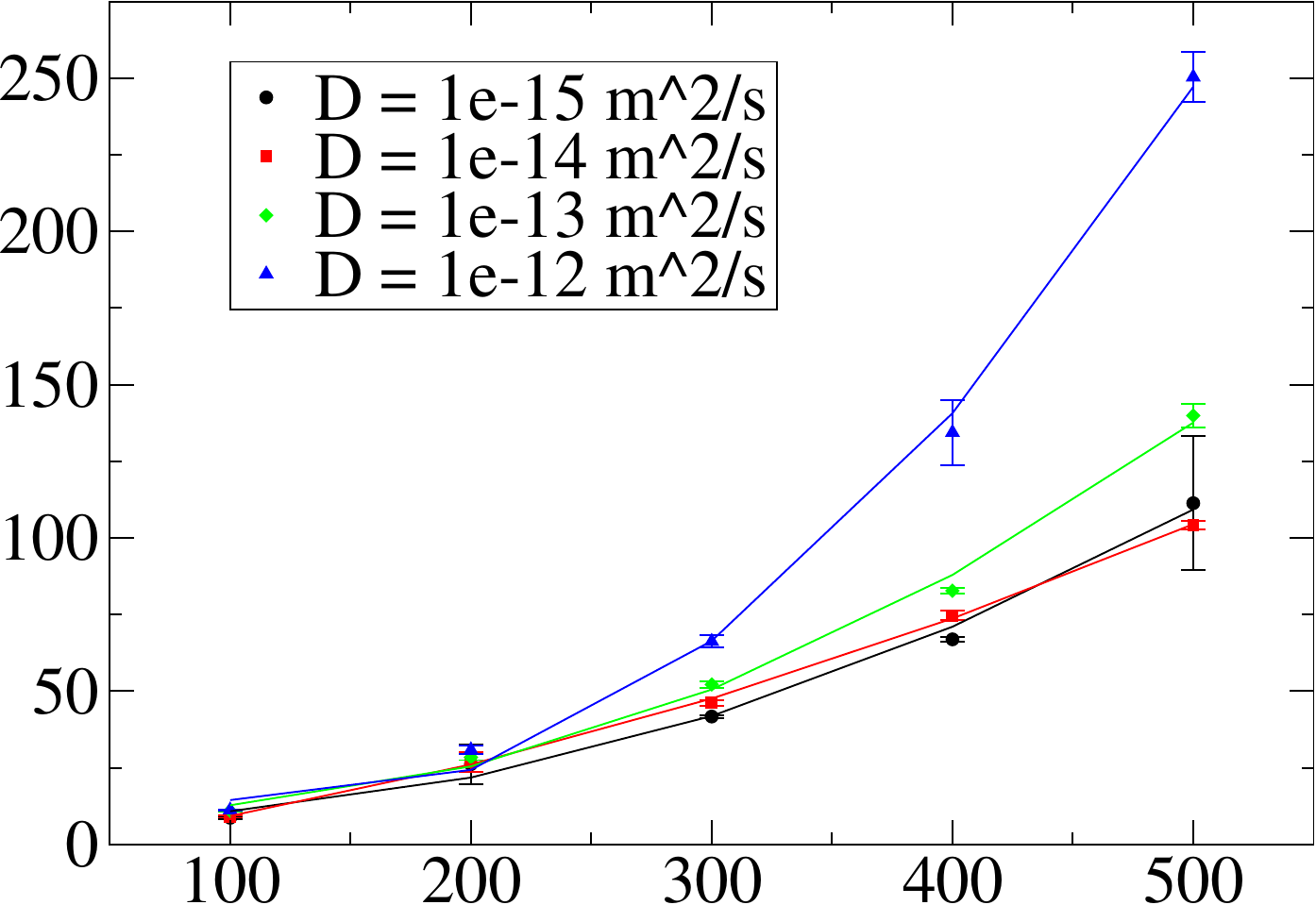}
 \caption{The running times in seconds depending on the number of particles per frame. The plots correspond to tracking radii 5,10,15 pixels from left to right. The error bars indicate the 95\% confidence intervals. Quadratic polynomials are fitted to the data.}
 \label{fig:runningtimes}
\end{figure}

The naive hypothesis for the running time coming from the theoretical bound of Ye~\cite{Ye91} for log-barrier interior point methods, i.e.~cubic in the number of variables, can be clearly dismissed. This is not surprising since the special sparse structure of our constraint matrix is unlikely to serve as a worst-case example. Instead, we found experimental support for the hypothesis that the running times scale quadratically w.r.t. the number of particles per frame. The rationale for this hypothesis is that the degree of the node of a point in a frame is proportional to the number of particles per frame. Thus, the number of arcs is quadratic in the number of particles per frame. Hence, the construction of the LP model takes quadratic time, which is the dominating part for low densities and slow molecules. However, with high densities and fast molecules, we see a turn-over to the LP-solver: the faster the molecules, the fewer arcs (because the shorter is the time a molecule stays within a fixed circle), but 
on the other hand the more connected the graph becomes because each molecule \emph{sees} more other molecules. We believe that the increasing correlation is responsible for the slow-down of the LP-solver and the higher running times for $D = 10^{-12}\Dunit$. This is supported by the observation that for small tracking radii the computations for fast molecules finish earliest while for the large tracking radius it takes more time than the others (see Fig.~\ref{fig:runningtimes}). Nevertheless, we stress that these running times are negligible
w.r.t.~the time necessary for preparing and executing such an experiment in reality. Thus, concerning computational resources, our approach is well suited for being applied in the lab.

\subsubsection{Determination of tracking accuracy for simulated data}

Tracking procedures have to meet several conditions to be suitable. Apart from practical
aspects such as tracking speed and memory consumption, the number of false positives is the
key factor which has to be minimized in order to obtain reliable results. 
Similarly to the definition of false positive and false negative w.r.t.~particle locations,
false positives in this context are connections between positions in different frames which have been set
even though the positions do not belong to the same molecule/particle. They can result in
severe errors in single molecule tracking and cause wrong interpretations of collected data.
Thus, the number of false positives should be kept as low as possible.
False negatives are connections between positions in different frames which are not
recognized by the tracking algorithm.

We counted the number of false positives and false negatives for movies of different diffusion
coefficients and signal-to-noise-ratios by comparing ground truth and analyzed
connections between points. The fraction of false positive
connections decreases from 13\% to 3\% as the S/N-ratio increases from 1 to 5.
False negatives particularly occur with fast moving molecules if the tracking radius is not chosen carefully.
The reason is that the probability of finding the destination within a radius of R is
\[
 P(R,t) = \int_0^R \frac{r}{2Dt} \exp\left( -\frac{r^2}{4Dt}\right) dr = 1 - \exp \left(-\frac{R^2}{4Dt}\right).
\]
Thus, picking the tracking radius too low yields biased false negatives and hence an underestimation of the diffusion coefficients. However, the tempting choice of an excessive tracking radius does not only require much more computational resources, but may also lead to an overestimation of the diffusion coefficients if there occur leaps in the tracks due to false positives (in particular with high particle densities).

We propose to choose the tracking radius such that a displacement is smaller with a probability of about 99\%. Since those leaps are easily determined in a post-processing steps, a repetition of the tracking with different radii in a feedback loop is feasible. In particular, allowing or disallowing connections between spots can be done efficiently with linear programming since the LP remains primal feasible or  dual feasible, respectively. Thus, in the former case, we shall use the primal simplex method for re-optimization and the dual simplex method in the latter case.

Positioning and tracking are the key steps in the determination of diffusion coefficients. To
distinguish the errors appearing in these steps, we analyzed four different cases shown in
Fig.~\ref{fig:experimental_paths}.
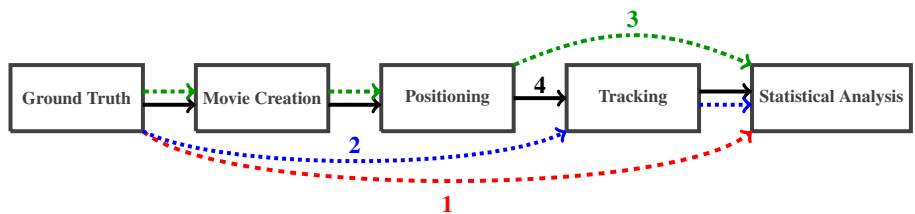
\begin{figure}[ht]
 \centering
\resizebox{\textwidth}{!}{
\definecolor{green}{RGB}{0,150,0}
\begin{tikzpicture}[scale=0.47]
  % frame
  \draw[darkgray,-,line width=2pt]{}   (0,0)
    -- (5,0)  -- (5,2.5) -- (0,2.5)  -- (0,0)
	node (n1) at (2.5,1.25) {\textbf{Ground Truth}};
  \draw[darkgray,-,line width=2pt]{}   (7,0)
    -- (12,0)  -- (12,2.5) -- (7,2.5)  -- (7,0)
	node (n1) at (9.5,1.25) {\textbf{Movie Creation}};
  \draw[darkgray,-,line width=2pt]{}   (14,0)
    -- (19,0)  -- (19,2.5) -- (14,2.5)  -- (14,0)
	node (n1) at (16.5,1.25) {\textbf{Positioning}};
  \draw[darkgray,-,line width=2pt]{}   (21,0)
    -- (26,0)  -- (26,2.5) -- (21,2.5)  -- (21,0)
	node (n1) at (23.5,1.25) {\textbf{Tracking}};
  \draw[darkgray,-,line width=2pt]{}   (28,0)
    -- (34,0)  -- (34,2.5) -- (28,2.5)  -- (28,0)
	node (n1) at (31,1.25) {\textbf{Statistical Analysis}};

  \draw[->, line width = 2pt] (5,1) -- (7,1);
  \draw[->, line width = 2pt] (12,1) -- (14,1);
  \draw[->, line width = 2pt,color = green, dashdotted] (5,1.5) -- (7,1.5);
  \draw[->, line width = 2pt,color = green, dashdotted] (12,1.5) -- (14,1.5);
  \draw[->, line width = 2pt] (19,1.25) -- (21,1.25)
	node (n1) at (20,1.75) {\textbf{\large{4}}};
  \draw[->, line width = 2pt, color = green, dashdotted] (19,2.5) .. controls (22.5,4) and (24.5,4) .. (28,2.5) node (n1) at (23.5,4.25) {\textbf{\large{3}}};
  \draw[->, line width = 2pt, color = red, dashed] (5,0) .. controls (8,-2.5) and (25,-2.5) .. (28,0) node (n1) at (16.5,-2.75) {\textbf{\large{1}}};
  \draw[->, line width = 2pt, color = blue, dotted] (5,0) .. controls (8,-1.5) and (18,-1.5) .. (21,0) node (n1) at (13,-0.5) {\textbf{\large{2}}};
  \draw[->, line width = 2pt] (26,1.5) -- (28,1.5);
  \draw[->, line width = 2pt, color = blue, dotted] (26,1) -- (28,1);

\end{tikzpicture}
}
 \caption{Steps from the simulated ground truth data to the analysis of the trajectories. To analyze the
accuracy of each single step different paths were considered (labeled 1 to 4).}
 \label{fig:experimental_paths}
\end{figure}

In case (1), the distribution of diffusion coefficient was directly calculated from
the ground truth tracks. Though all tracks were created with respect to fixed diffusion coefficients, the calculation yields peaked distributions around the true values because of the finite number of sample points in each track.

In a second set of analysis (case 2) the same ground truth positions were tracked
using our polyhedral model solved with CPLEX. Good results were obtained except for fast molecules and high densities. That is, the
probability of foreign spots moving into the tracking range of another molecule is
too high, and thus the tracking algorithm in general returns diffusion coefficients lower than
the real value.

The third set of analysis, case 3, allows for an investigation of the influence of positioning
inaccuracies on the distribution of diffusion coefficients. Movies were constructed from the
ground truth positions with different S/N-ratios. Our positioning algorithm was applied to
these movies and, where possible, the positions matched to the positions of the ground truth
tracks. With the determined positions of each track, a diffusion coefficient was obtained. The
distribution of diffusion coefficients of these tracks resembles the distributions of the ground
truth tracks with the exception of low diffusion coefficients with low S/N-ratios where the
poor localization accuracy results in a seemingly higher diffusion coefficient than simulated.

Case 4 describes the procedure which is applied for real movies to determine single molecule
diffusion coefficients. Spots in movies are positioned, the positions tracked and a diffusion
coefficient calculated from these tracks. For high S/N-ratios, the obtained distributions are
similar to the simulated distributions. Two trends can be observed in particular at low S/N-
ratios:
\begin{enumerate}[(i)]
 \item molecules with very low diffusion coefficients tend to be analyzed as being faster as they
were simulated, and
\item analysis of the motion of very fast molecules in average results in a lower
diffusion coefficient as the ground truth data.
\end{enumerate}
The former observation can be explained by the
poorer localization accuracy at low S/N-ratios. This inaccuracy resembles diffusion and thus a
low diffusion coefficient will be assigned even to immobile molecules. The localization
accuracy determines the lowest diffusion coefficient which can be determined by the
corresponding experimental settings.

\subsubsection{Tracking of blinking data}

In the previous subsection, points were missing due to false negatives in the localization at low S/N-ratios.
However, even with perfect localization missing points may occur naturally in real-world experiments because
the fluorescence intensity of single molecules is typically not constant, but shows blinking
behavior due to photochemical or photophysical quenching processes \cite{Cichos_2007_15}.
The lengths of on- and off-times typically show a power law distribution \cite{Cichos_2007_15}. In order to simulate
blinking behavior, we generated on- and off-times for our simulated tracks using the
following procedure. At the beginning a molecule was u.a.r. set as on or off. The number of
frames remaining in this state was determined randomly from a probability distribution with $P(t) \sim \left( t/\tau\right)^\alpha$.
We chose realistic values for $\tau$ and $\alpha$, i.e.~$\tau = 1$s and $\alpha = -2$, respectively.

\begin{figure}[ht]
 \centering
 \includegraphics[width= 0.49 \textwidth,page=7]{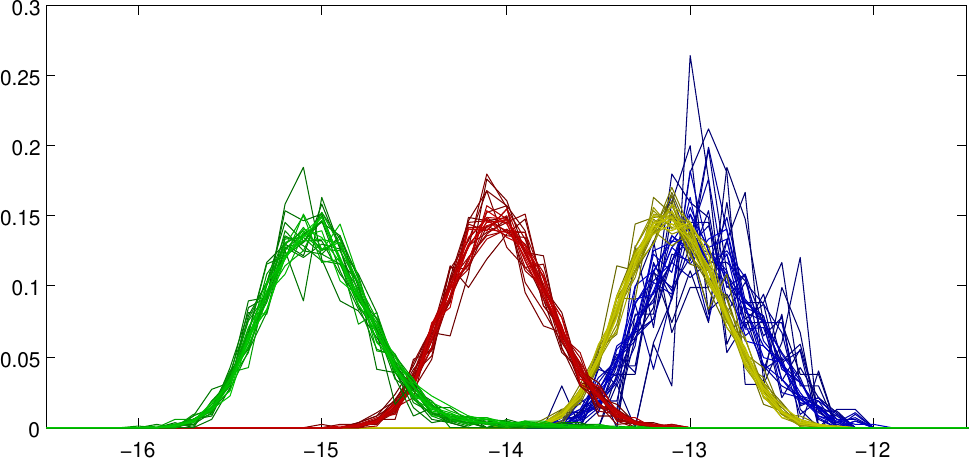} \hfill
 \includegraphics[width= 0.49 \textwidth,page=8]{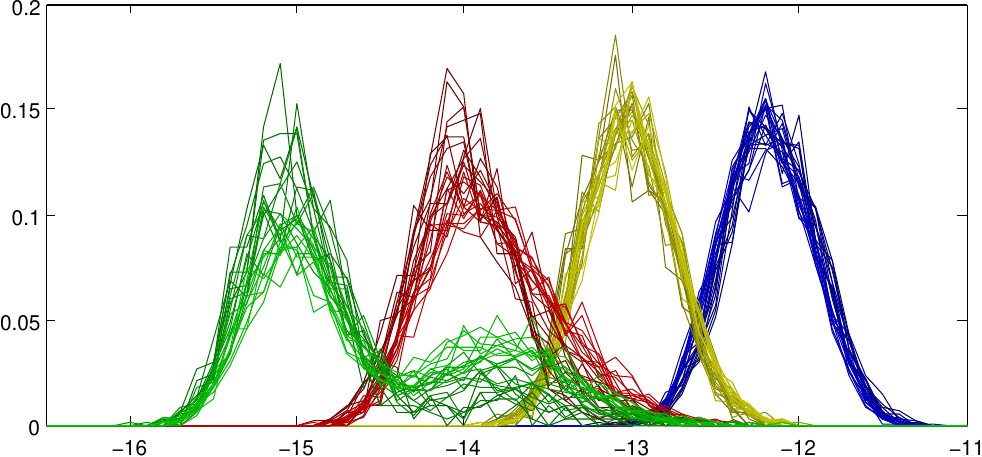}
 \caption{Distributions of diffusion coefficients obtained after tracking of blinking data for a maximum step
length of 5 (left) and 15 (right). The scale for the x-axis is $\log_{10}( D \cdot \text{m}^{-2}\text{s})$.}
 \label{fig:9}
\end{figure}

The analyzed distributions of diffusion coefficients for tracking of blinking ground truth data
are shown in Fig.~\ref{fig:9}. For a tracking radius of 5 pixels, the distributions of diffusion
coefficients resemble the ground truth data except for fast molecules where a tracking radius of 5 pixels is not sufficient as discussed before.
A tracking radius of 15 pixels yields good results for fast diffusing
molecules, but for slower ones, the distributions have a long tail
or even a second band at higher D values (see Fig.~\ref{fig:9} (right)). The deviation of distribution
from the ground truth distribution is caused by tracks which include at least one large jump
from one molecule to another one which results in a significant increase of the track radius. However, these situations can be recognized easily or even automatically by an outlier detection algorithm. As said before, we propose to integrate such a post-processing in a feedback loop to deal with the such situations especially when heterogeneous ensembles with slow and fast molecules are observed.

\subsection{Real-World Data}

The real-world data were obtained from single molecule fluorescence widefield experiments during the bulk radical polymerization of styrene to polystyrene. The motion of single perylene diimide fluorophores was observed at various monomer-to-polymer-conversions and thus different viscosities which allowed us to probe a broad range of diffusion coefficients. The interested reader is referred to~\cite{Stempfle_2012_59} for more details.

Before this project, the tracks were constructed semi-manually due to the lack of a satisfying alternative. That is, only a simple search in the neighborhood of the points was performed automatically as long as there were no ambiguities, i.e.~only one localized point within the tracking radius and no competition among potential predecessors. In the case when the automatic continuation of the tracks fails, the user was presented with 10 consecutive frames of the movie with the options to select a successor among the alternatives, to introduce a new spot that was not detected by the localization, or to end the track. Needless to say that this was a tedious task, which took several working days to complete the tracking of a 5-minute-movie with high particle density. The advantage of this method is that the human expert maintains the full control over the process and the pattern recognition capabilities of the human brain is leveraged to resolve situations in which the image processing tools fail. On the other
hand, these possibilities are also a disadvantage as the user might introduce systematic errors in the data and it is unlikely that a repetition of the task yields exactly the same results.

It remains to show that our automatic method not only works for realistic data but also in the real-world. To this end, we compare the average diffusion coefficients obtained from 6 movies by manual and automatic tracking:

\begin{center}
\begin{tabular}{|l|rrrrrr|c|}
\hline
manual     &  0.019 &   0.053 &   0.126 &   0.537 & 1.166 &  4.864 & \multirow{2}{*}{$\cdot 10^{-13}\Dunit$} \\
automatic  &  0.023 &   0.054 &   0.132 &   0.509 & 1.054 &  4.372 & \\
\hline
\end{tabular}
\end{center}

\bibliographystyle{ieeetr}
\bibliography{sea2013}

\end{document}